\documentclass[letterpaper]{article} 
\usepackage{aaai24}  
\usepackage{times}  
\usepackage{helvet}  
\usepackage{courier}  
\usepackage[hyphens]{url}  
\usepackage{graphicx} 
\usepackage{natbib}  
\usepackage{caption} 
\usepackage{algorithm}

\usepackage[noend]{algpseudocode}
\usepackage{algpseudocode}
\usepackage{amsmath}
\usepackage{amsthm}
\usepackage{diagbox}
\usepackage{subcaption}
\usepackage{arydshln}
\usepackage{array}
\usepackage{multirow}
\usepackage{tabularx}
\usepackage{multirow}
\usepackage{caption}
\usepackage{float}
\usepackage{colortbl}
\usepackage{bm}
\usepackage{amssymb}
\usepackage[english]{babel}
\newtheorem{theorem}{Theorem}
\newtheorem{lemma}[theorem]{Lemma}
\newtheorem{proposition}{Proposition}

\usepackage{newfloat}
\usepackage{listings}
\DeclareCaptionStyle{ruled}{labelfont=normalfont,labelsep=colon,strut=off} 
\floatstyle{ruled}
\newfloat{listing}{tb}{lst}{}
\floatname{listing}{Listing}
\title{Fair Multivariate Adaptive Regression Splines\\ for Ensuring Equity and Transparency
}

\author {
    Parian Haghighat\textsuperscript{\rm 1},
    Denisa G\'andara\textsuperscript{\rm 2},
    Lulu Kang\textsuperscript{\rm 3},
    Hadis Anahideh\thanks{Corresponding Author}\textsuperscript{\rm 1}
}
\affiliations {
    \textsuperscript{\rm 1}University of Illinois Chicago \\
    \textsuperscript{\rm 2}University of Texas at Austin\\
     \textsuperscript{\rm 3}University of Massachusetts Amherst\\
    phaghi3@uic.edu, denisa.gandara@austin.utexas.edu, lulukang@umass.edu, hadis@uic.edu
}

\usepackage{bibentry}

\begin{document}

\maketitle

\begin{abstract}
Predictive analytics is widely used in various domains, \mbox{including education}, to inform decision-making and improve outcomes. However, many predictive models are proprietary and inaccessible for evaluation or modification by researchers and practitioners, limiting their accountability and ethical \mbox{design}. Moreover, predictive models are often opaque and \mbox{incomprehensible} to the officials who use them, \mbox{reducing} their trust and utility. Furthermore, predictive \mbox{models} may \mbox{introduce} or exacerbate bias and inequity, as they have done in many sectors of society. Therefore, there is a need for \mbox{transparent}, interpretable, and fair predictive models that can be easily adopted and adapted by different stakeholders. In this paper, we propose a fair predictive model based on \mbox{multivariate} adaptive regression splines (MARS) that \mbox{incorporates} fairness measures in the learning process. MARS is a non-parametric regression model that performs feature selection, handles non-linear relationships, generates interpretable decision rules, and derives optimal splitting \mbox{criteria} on the variables. Specifically, we integrate fairness into the knot optimization algorithm and provide theoretical and empirical evidence of how it results in a fair knot \mbox{placement}. We apply our \textit{fair}MARS model to real-world data and demonstrate its effectiveness in terms of accuracy and \mbox{equity}. Our paper contributes to the advancement of \mbox{responsible} and ethical predictive analytics for social good.
\end{abstract}

\section{Introduction}

Predictive analytics has made remarkable progress in the past few decades across diverse sectors, ranging from \mbox{education} \cite{yu2020towards} to healthcare \cite{stiglic2020interpretability}, as a potent tool for guiding decision-making processes and augmenting overall outcomes. Predictive models have the potential to revolutionize practices by offering foresight into future developments. However, a recurring challenge lies in the propriety nature of many of these models, which bars researchers and practitioners from evaluating, adapting, or optimizing these models to align with ethical considerations and accountability standards.

This limitation not only restricts the potential for improvement but also undermines the principles of transparency and fairness in design for high-stakes domains such as education. Education is a complex and dynamic system that requires accounting for the diversity, variability, and uncertainty of the data and the context \cite{educsci14020136}. Moreover, predictive models in education need to be transparent, interpretable, and fair to ensure accountability, trustworthiness, and ethical design.

Transparency and interpretability are the abilities to access, understand, and explain the data, methods, and \mbox{assumptions} behind the predictive models.  They are essential for ensuring accountability, trustworthiness, and ethical design of predictive models. Fairness is the ability to prevent or mitigate bias and discrimination against certain groups or individuals based on their sensitive attributes. Fairness is crucial for ensuring equity, justice, and social good of predictive models.

Several methods have been proposed to enhance the interpretability of machine learning (ML) models, such as feature selection \cite{Khaire_Dhanalakshmi_2022}, feature \mbox{importance} \cite{linardatos2020explainable}, decision rules \cite{wang2017bayesian}, and decision trees \cite{bohm2010profiling}. These methods can be broadly classified into two categories: intrinsic and post-hoc. Intrinsic methods aim to design interpretable models from scratch, by imposing constraints or regularization on the model structure or complexity \cite{zambaldi2018relational}. Post-hoc methods aim to explain black-box models after they are trained, extracting or approximating the model’s behavior using simpler or more transparent models \mbox{\cite{pmlr-v48-zahavy16}}.

To ensure fairness in ML, researchers have developed various fairness metrics and methods for \mbox{measuring} and enforcing fairness in the models. For example, \mbox{statistical} parity  ~\cite{chouldechova2017fair,simoiu2017problem} requires that the model's predictions are independent of sensitive attributes (e.g., race), and equal opportunity \cite{hardt2016equality} requires that the model's true positive rates are equal across different groups.
Many bias mitigation algorithms have been \mbox{proposed} to achieve these and other fairness criteria, especially for \mbox{classification} problems, where the goal is to predict a binary outcome \cite{berk2017convex, chouldechova2017fair,hardt2016equality,kilbertus2017avoiding}. However, for regression problems, where the goal is to model a continuous outcome variable, such as a real-valued score, there are \mbox{fewer} \mbox{studies} on fairness-aware models. Some of the existing works include fair regression models \cite{agarwal2019fair,berk2017convex,komiyama2018nonconvex,jabbari2016fairness,Calders2013ControllingAE,johnson2022impartial} and fair regression trees \cite{Aghaei_Azizi_Vayanos_2019}. These works aim to balance the trade-off between accuracy and fairness in regression tasks and to provide interpretable and transparent models that can be easily adopted and adapted by different stakeholders.

Regression tasks, such as predicting students' academic performance, college success, or dropout risk, hold broad applicability and social impact in education. A fair and interpretable regression model that can capture the complexity of education system algorithms can promote a just and equitable society, where students' opportunities are not affected by demographic characteristics. 

In this paper, we present the development of a fair regression model based on multivariate adaptive regression splines (MARS), which incorporates fairness measures in the learning process. MARS, originally introduced by Friedman \cite{friedman1991multivariate}, stands out as a powerful and efficient method for flexible regression modeling of high-dimensional data. It employs a non-parametric, non-interpolating approach and consists of two main parts: forward selection and backward elimination.

In the forward selection phase, MARS divides the solution space into intervals and fits spline basis functions to each interval. These basis functions are added incrementally, minimizing the lack-of-fit, until a user-specified maximum number of basis functions is reached. The backward elimination phase reduces model complexity by removing the least contributing basis functions. This backward elimination process helps in selecting a subset of the most important basis functions, resulting in a more parsimonious and interpretable model. Both the forward selection and backward elimination phases utilize generalized cross-validation (GCV) as a basis selection criterion, ensuring optimal configuration of basis functions. \textit{fair}MARS is a type of intrinsic method for enhancing interpretability and can produce transparent and fair predictions that are easy to understand and communicate.

MARS outperforms other statistical models, particularly decision tree regression, in several key aspects. Unlike decision tree regression, which often produces piecewise constant or step-like functions, MARS can create smooth curves that more accurately represent the underlying relationships. This flexibility enables MARS to capture intricate patterns and enhance the model's generalization to unseen data. Additionally, the interpretability of MARS facilitates understanding and insight into the underlying relationships between variables, making it desirable in various applications where transparency and comprehensibility are crucial. Moreover, MARS incorporates a built-in variable selection mechanism, which is especially valuable for large-scale problems with a substantial number of attributes. This feature enables MARS to automatically select the most relevant variables, ensuring a parsimonious model and mitigating the challenge of feature importance selection.

Our proposed \textit{fair}MARS incorporates fairness in the learning process in two ways: first, it modifies the knot optimization algorithm of the forward and backward pass of MARS, which determines the optimal location and number of knots for each predictor variable, by adding fairness constraints to the objective function. Second, it integrates fairness into the loss function minimization for estimating coefficients after selecting basis functions. These modifications influence the learning algorithm, resulting in a fair knot placement and fair coefficient estimation.


\textit{fair}MARS can produce smooth and flexible curves that accurately represent the underlying relationships between predictor and outcome variables across different sensitive subgroups. Additionally, it generates interpretable decision rules by deriving optimal and fair splitting criteria on the variables and by pruning basis functions that contribute the least to both fairness and accuracy during the forward and backward phases.


\textit{fair}MARS stands out compared to the fair decision tree regression model by \citet{Aghaei_Azizi_Vayanos_2019} in terms of accuracy, flexibility, and interpretability. Unlike fair decision tree, which binarizes continuous variables in a pre-processing step, thereby undermining interpretability, \textit{fair}MARS does not require such data transformation. It generates interpretable decision rules, providing transparency, which also distinguishes it from opaque black-box models \cite{cruz2023fairgbm}. Additionally, \textit{fair}MARS adeptly handles non-linearity, a capability lacking in linear models such as fair linear regression \cite{agarwal2019fair}.

The following sections of the paper are organized as follows: we begin with a review of closely related work on fair regression models in the ``Related Work'' section, and highlight the differences from our approach. Following that, the ``Methodology'' section introduces our \textit{fair}MARS method, which builds upon the foundation of MARS and extends its capabilities to prioritize fairness in knot selection and coefficient estimations. Finally, empirical evidence of the \mbox{effectiveness} of \textit{fair}MARS and its comparison with existing baselines is presented in the ``Experiments'' section.

\section{Related Work}

Several works have proposed fairness-aware ML and data mining methods that aim to prevent algorithms from discriminating against protected groups \cite{kleinberg2018algorithmic}. The literature has come to an impasse as to what constitutes explainable variability as opposed to discrimination \cite{johnson2022impartial}. Different fairness measures have been proposed to capture different notions of fairness including demographic parity ~\cite{calders2009building,dwork2012fairness,kamiran2009classifying} and equalized odds {\cite{hardt2016equality,zafar2017fairness}.\\
Algorithmic fairness can be achieved by intervention in pre-processing, in-processing, or post-processing strategies \cite{friedler2019comparative}. Pre-process approaches involve the fairness measure in the data preparation step to mitigate the potential bias in the input data and produce fair outcomes \cite{feldman2015certifying,kamiran2012data,calmon2017optimized}. In-process approaches incorporate fairness in the design of the algorithm to generate fair outcomes \cite{kamishima2012fairness,zemel2013learning,zafar2015fairness,hu2020fair}. Post-process
methods \cite{feldman2015certifying,zehlike2017fa} manipulate the outcome of the algorithm to mitigate the unfairness of the algorithm outcome for the decision-making process.\\
Besides fairness, another important aspect of predictive analytics is interpretability. Interpretability can help users understand how and why predictions are made and what actions they can take based on them. There are different types of interpretable ML models including intrinsic (such as decision trees, decision rules, and linear regression) and post-hoc methods (such as feature importance, partial dependence plots, and accumulated local effects) \cite{rudin2022interpretable}.
Several works have addressed the issue of interpretability in domains where predictions have significant social impacts, such as health care \cite{stiglic2020interpretability}, criminal justice, education, and finance \mbox{\cite{makhlouf2021applicability}}.\\ 
Few recent works have explored the intersection of fairness and interpretability in ML, and proposed methods and tools to achieve both objectives simultaneously. For example, \citet{agarwal2021trade} proposed a framework for building interpretable ML models that can also satisfy different fairness criteria, based on monotonicity constraints. \citet{fu2020fairness} proposed a method for generating fair and explainable recommendations using knowledge graphs, and introduced a novel fairness measure for recommendation systems. \citet{cabrera2019fairvis} presented a visual analytics tool for discovering and mitigating intersectional bias in ML models. \citet{panigutti2021fairlens} presented a methodology and a tool for auditing datasets for discrimination and bias, and provides visualizations and explanations for the discrimination analysis. Another example is the work of \citet{Aghaei_Azizi_Vayanos_2019}. They proposed a mixed-integer optimization framework for learning optimal and fair decision trees that can prevent disparate treatment and/or disparate impact in non-discriminative decision-making.\\
However, most of these works focus on classification problems and do not address the challenges of regression problems. Regression problems are common in predictive analytics, especially in education, where the outcomes are often continuous or ordinal variables, such as test scores, grades, or graduation rates. Fairness-aware regression methods need to account for the distribution of the outcome variable across different groups and balance the trade-off between accuracy and equity. A few works have attempted to design fair regression models that satisfy different fairness criteria \cite{agarwal2019fair,berk2017convex,komiyama2018nonconvex,jabbari2016fairness,Calders2013ControllingAE,johnson2022impartial}. However, these methods either rely on parametric assumptions that may not hold in practice or are vulnerable to adversarial attacks that can exploit the fairness constraints. Moreover, these methods do not provide interpretable explanations for their predictions, which may limit their trustworthiness and accountability. Therefore, there is a need for a transparent and robust fair regression method that can handle non-linearities and interactions, generate interpretable decision rules, and derive optimal splitting criteria on the \mbox{variables.}

In this work, we propose a transparent and \mbox{accountable} predictive model based on MARS. 
MARS is a non-parametric regression model that performs feature selection, handles non-linear relationships, generates interpretable decision rules, and derives optimal splitting criteria on the variables. MARS has several advantages over other regression methods in terms of transparency and interpretability and does not make any assumptions about the functional form of the underlying relationship. MARS can capture non-linearities and interactions without prior transformations. 

\section{Methodology}
\paragraph{Problem Setup.} Let $n$ be the number of samples in the training set. 
The $i$th data point is denoted by $(\mathbf s_i, \mathbf x_i, y_i)$, where $\mathbf s_i\in \mathcal{R}^{d_s}$ is the sensitive attributes of $d_s$ dimensions (e.g., race), $\mathbf x_i\in \mathbb{R}^{d_x}$ is the non-sensitive features of $d_x$ dimensions, and $y_i\in \mathbb{R}$ is the observed response.
In this paper, we only consider the scenario of univariate continuous response. The training dataset is $\{(\mathbf s_i, \mathbf x_i, y_i)\}_{i=1}^n$. 
The goal is to learn a regression model that maps the set of input attributes $(\mathbf s,\mathbf x)$ to the response $\bm y$.  For a comprehensive list of notations, please refer to Table \ref{tab:notations}.

\begin{algorithm}[t]
    \small
    \caption{MARS— Forward Stepwise}
    \label{alg:mars_forward_fair}
    \begin{algorithmic}[1]
    \Require{$\mathbf{x}$, $y$, $M_{\max}$
    }
    \State $B_1(\mathbf{x}) \leftarrow 1$; $M \leftarrow 2$
    \While{$M > M_{\max}$: lof* $\leftarrow \infty$}
        \For{$m=1$ to $M-1$}
            \For{$v \notin\left\{v(q, m) \mid 1 \leq q \leq N_m\right\}$}
                \For{$k \in\left\{x_{v j} \mid B_m\left(\mathbf{x}_j\right)>0\right\}$}
                    \State $g^{\prime} \leftarrow \sum_{i=1}^{M-1} \beta_i B_i(\mathbf{x})+\beta_M B_{m^*}(\mathbf{x}) x_v+\beta_{M+1} B_{m^*}(\mathbf{x})\left(x_v-k\right)_{+}$
                    \State \textbf{Fast Updating Formula:}
                    \State $\bar{y}=0$
                    \State $C1= \sum \limits_{k \leq x_{v q}<u}\left(y_q-\bar{y}\right) B_{m q}\left(x_{v q}-k\right)$
                    \State $C2=(u-k) \sum \limits_{x_{v q}\geq u}\left(y_q-\bar{y}\right) B_{m q}$
                    \State $c_{M+1}(k) \leftarrow c_{M+1}(u)+ C1+ C2 $
                    \State $\text{LOF}(g^{\prime})= \sum_{q=1}^N\left(y_q-\bar{y}\right)^2-\sum_{i=1}^{M+1} \beta_i\left(c_i\right)$

                    \State lof $\leftarrow \min_{\beta_1, ..., \beta_{M+1}} \mathrm{LOF}(g^{\prime})$
                    \If{lof $<$ lof*} lof* $\leftarrow$ lof; $m^* \leftarrow m$
                    \EndIf
                \EndFor
            \EndFor
        \EndFor
        \State $B_M(\mathbf{x}) \leftarrow B_{m^*}(\mathbf{x})\left[+\left(x_{v^*}-k^*\right)\right]_{+}$
        \State $B_{M+1}(\mathbf{x}) \leftarrow B_{m^*}(\mathbf{x})\left[-\left(x_{v^*}-k^*\right)\right]_{+}$
        \State $M \leftarrow M + 2$
    \EndWhile
    \State \Return Fitted MARS Forward Model
    \end{algorithmic}
\end{algorithm}

\begin{table}[!ht]
    \small
    \begin{tabular}{c|c}
        \textbf{Notation} & \textbf{Definition} \\
        \hline
        $B_m$ & Current Basis Functions \\
        $B_{m^*}$ & Selected Basis Function \\
        $\beta_i$ & Coefficients of $B_m$ \\
        $i$ & Index for Data Points\\
        $j$ & Predictor Variable Number \\
        $k$ & Eligible Knot Value \\
        $k^*$ & Best Knot Value \\
        $lof$ & Linear Regression of $B_m$ Set \\
        $M$ & Iteration Number (Regions) \\
        $m$ & Basis Number \\
        $m_j$ & Number of Unique Values of $x_j$\\
        $n$ & Number of Samples in the Training Set\\
        $N_m$ & Maximum Interaction Order of the Basis Functions \\
        $q$ & Interaction Order of the $B_m$ \\
        $\mathbf{s}_i$ & Sensitive Attributes of the $i$-th Data Point \\
        $\mathbf{x}_i$ & Non-sensitive Features of the $i$-th Data Point \\
        $y_i$ & Observed Response of the $i$-th Data Point \\
    \end{tabular}
    \caption{Table of Notations}
    \label{tab:notations}
\end{table}

MARS, introduced by \citet{friedman1991multivariate}, is a powerful regression technique that can capture complex relationships in the data. The MARS model is based on basis functions, which are either univariate or multivariate interaction terms, built from truncated linear functions (also known as hinge functions) that bend at knot locations of the form $[x-k]{+}$ or $[k-x]{+}$, where $\left[ u \right]_{+}= \max{(0,u)}$, $x$ is a single independent variable, and $k$ is the corresponding univariate knot.

A knot is the location of an intersection between two segments of a spline and therefore is an important concept associated with MARS. Each dimension value of the training dataset is an eligible knot for the corresponding independent variable $x_1,\dots, x_{d_x}$.
Let $\mathcal{K}j = \{k_{j1}, k_{j2}, ..., k_{jm_j}\}$ be the set of candidate knots for the $j$-th predictor variable, where $m_j$ is the number of unique values of $x_j$. Let $\mathcal{M} = \{\beta_0, \beta_1, ..., \beta_q\}$ be the set of coefficients for the MARS model, where $q$ is the number of basis functions.
Let $\mathrm{RSS}(\mathcal{M}) = \sum_{i=1}^n (y_i - \hat{y}_i)^2$ be the residual sum of squares for the MARS model, where $\hat{y}_i = \beta_0 + \sum_{k=1}^q \beta_k B_k(x_i)$ is the predicted value for the $i$-th observation and $B_k(x_i)$ is the $k$-th basis function.
Initialize $\mathcal{M} = {\beta_0}$ and $\mathrm{RSS}(\mathcal{M}) = \sum_{i=1}^n (y_i - \bar{y})^2$, where $\bar{y} = \frac{1}{n} \sum_{i=1}^n y_i$ is the mean of response variable.

\paragraph{Knot Optimization.}
This algorithm is a procedure that aims to find the optimal location and number of knots for each predictor variable in the MARS model. The algorithm starts with a large number of candidate knots at each unique value of the predictor variable, and then iteratively adds or deletes knots based on the change in the $\mathrm{RSS}$ criterion. 
The algorithm stops when no further improvement can be made by adding or deleting knots.

For each predictor variable $x_j$ and each candidate knot $k_{jl} \in \mathcal{K}_j$, compute the change in $\mathrm{RSS}$ (i.e., the contribution of a basis function) by adding a pair of basis functions $B_{q+1}(x) = [x_j - k_{jl}]_+$ and $B_{q+2}(x) = [k_{jl} - x_j]_+$ to the model. The calculation is of the form:
$$\Delta \mathrm{RSS}(j,l) = \mathrm{RSS}(\mathcal{M}) - \mathrm{RSS}(\mathcal{M} \cup {\beta_{q+1}, \beta_{q+2}})$$

The algorithm finds the predictor variable and knot that gives the maximum decrease in $\mathrm{RSS}$: $(j^*, l^*) = \arg\max_{j,l} \Delta \mathrm{RSS}(j,l)$.
If $\Delta \mathrm{RSS}(j^*, l^*) > 0$, then the corresponding basis functions and coefficients is added to the model, $\mathcal{M} = \mathcal{M} \cup {\beta_{q+1}, \beta_{q+2}}$, and $(\mathcal{M})$ is updated.\\ 

The MARS algorithm consists of two essential steps: forward selection and backward elimination. In the forward selection procedure (Algorithm \ref{alg:mars_forward_fair}),  MARS adds candidate basis functions selected by knot optimization in pairs for different dimensions incrementally until reaching the maximum number of basis functions $M_{\max}$. Note that $N_m$ \mbox{denotes} the maximum interaction order of the basis functions within the MARS algorithm. To prevent overfitting, the backward elimination procedure continuously removes the least effective basis at each step. Both steps employ Generalized Cross-Validation (GCV) to select the best subset of basis functions, striking a balance between model complexity and accuracy.

In section 3.9 of \citet{friedman1991multivariate}, the authors employ Cholesky decomposition to solve normal equations as a means to accelerate updates within the stepwise forward procedure  — a crucial element contributing to the efficiency of the model. This fast update formula arranges the candidate knots of each dependent variable in descending order (equation \ref{eq:knot_order}) based on their impact on the model's fit.
\begin{equation}\label{eq:knot_order}
\begin{aligned}
    {[x-k]}_+-{[x-u]}_+ &= \begin{cases}
        0 & \text{if } x \leq k \\
        x-k & \text{if } k < x < u \\
        u-k & \text{if } x \geq u
    \end{cases}
\end{aligned}
\end{equation}
With the knots sorted, the algorithm proceeds iteratively (Algorithm \ref{alg:mars_forward_fair}, line 11), and solves equations 2,3,4 for each eligible $k$, for every $v$, and for all basis functions $m$ and for all iterations $M$ in algorithm 1, adding the best knot $k^*$ and its corresponding basis function $B_{m^*}$ to the model.

\begin{equation}
\mathbf{V} \mathbf{\beta}=\mathbf{c}
\end{equation}

\begin{equation}
\small
V_{i j}=\sum_{q=1}^n B_j\left(\mathbf{x}_q\right)\left[B_i\left(\mathbf{x}_q\right)-\bar{B}_i\right]
\end{equation}

\begin{equation}
c_i=\sum_{k=1}^n\left(y_q-\bar{y}\right) B_i\left(\mathbf{x}_q\right)
\end{equation}

Upon identifying $k^*$, the algorithm efficiently determines the best coefficient $\beta_{m^*}$ for the new basis function by passing the knot location to the LOF function (Algorithm \ref{alg:mars_forward_fair}, line 12). This optimized process of knot and coefficient selection avoids repetitive matrix computations, significantly enhancing the computational efficiency of the MARS algorithm, particularly for larger datasets. However, this criterion does not consider fairness which is important for societal applications. To address this issue, one may need to modify the criterion or add a regularization term that penalizes unfairness or other undesirable outcomes.

\paragraph{Fairness Considerations.} The $\mathrm{RSS}$ is a measure of how well the MARS model fits the data. The lower the $\mathrm{RSS}$, the better the fit. The contribution of a basis function or a knot selection on $\mathrm{RSS}$ is the amount of reduction in $\mathrm{RSS}$ that is achieved by adding or deleting that basis function or knot selection from the model. The higher the reduction in $\mathrm{RSS}$, the more important that basis function or knot selection is for improving the fit.

To incorporate fairness in the learning procedure of MARS, we need to first define the fairness metric in this context. There are many possible definitions of fairness, such as disparate treatment and disparate impact \cite{barocas2016big, zafar2017fairness, kamiran2012data}, equalized odds, equal opportunity \cite{hardt2016equality}. Each definition captures a different aspect of fairness and may have different implications. The choice of fairness also depends on the setting of the problem (regression/classification). Note that most of the existing notions are mainly designed for classification settings. As our focus in this paper is on regression problems, we consider a model fair if it does not discriminate or favor any group over another based on their sensitive attributes. For example, a fair regression model should not predict higher GPAs for white students as opposed to black students, under similar conditions.

To measure the fairness of the MARS model, we calculate the average error for each subgroup and compare them with each other, specifically through their absolute error differences, as it is a suitable notion of fairness for regression problems involving continuous outcomes and sensitive attributes. The error metric is a way of quantifying the disparity in prediction errors or losses between different groups defined by a sensitive attribute, such as race or gender.

To calculate the absolute error difference metric, we first define the prediction error or loss function for the regression task. A common choice is the residual sum of squares ($\mathrm{RSS}$), measuring the squared difference between actual and predicted outcomes, formulated as $\mathrm{RSS} = \sum_{i=1}^n (y_i - \hat{y}_i)^2$. Next, we divide the dataset into subgroups based on the sensitive attribute. For instance, if there are $n_f$ female examples and $n_m$ male examples, the average $\mathrm{RSS}$ for each subgroup is computed as follows: $\mathrm{RSS}_f = \frac{1}{n_f} \sum_{i=1}^{n_f} (y_i - \hat{y}_i)^2$ and $\mathrm{RSS}_m = \frac{1}{n_m} \sum_{i=1}^{n_m} (y_i - \hat{y}_i)^2$.

Finally, the absolute error difference metric is calculated by \mbox{finding} the absolute difference between the average $\mathrm{RSS}$ of each subgroup. For a scenario with two subgroups based on gender, the absolute error difference metric is given by $ |\mathrm{RSS}_f - \mathrm{RSS}_m|$. This metric serves as a measure of the discrepancy in error between groups resulting from the regression model. A lower absolute error difference signifies a fairer model, whereas a higher value indicates a more biased model. Essentially, our unfairness metric reflects how well the MARS model predicts the outcomes for each subgroup,  with a larger difference indicating greater accuracy for one subgroup over another.
By minimizing the average absolute $\mathrm{RSS}$ difference, we aim to achieve demographic parity and ensure fair and equal treatment for all subgroups.

Absolute error difference is preferable to other notions of fairness, such as disparate impact or group loss, because it captures the relative fairness between groups, rather than the absolute fairness for any group. Additionally, it aligns seamlessly with the non-linear and non-parametric nature of MARS. Moreover, when dealing with continuous outcomes that are sensitive to small changes, absolute error difference is more advantageous compared to disparate impact or group loss. For example, if the outcome is the income of workers, absolute error difference would measure the absolute difference in the average prediction error or loss between workers of different races or genders, while the disparate impact would measure the difference in the percentage of workers who are predicted to have a high income between workers of different races or genders, and group loss would measure the highest prediction error or loss among workers of different races or genders. Absolute error difference may be more appropriate than disparate impact or group loss in this case because it reflects the continuous and granular nature of the outcome, and because it does not impose arbitrary thresholds or levels for defining favorable or unfavorable outcomes.

We can incorporate fairness into the MARS training procedure in two ways. First, we can select knots that reduce both the overall $\mathrm{RSS}$ and the $\mathrm{RSS}$ difference of subgroups in the knot selection step. This way, we can achieve fairness by choosing knots that improve the model fit (accuracy) and balance the error rates for different subgroups (fairness). Second, we can incorporate fairness in the lack of fitness (LOF) function in the forward step, where the coefficients are estimated. We can convert the linear least-squares into a weighted least-squares to reflect the proportional distribution of observations within each subgroup. By minimizing the LOF function, we can find the optimal coefficients that ensure fair representation across different subgroups.
\subsection{Fairness-Aware Knot Selection}
To add the $\mathrm{RSS}$ difference of subgroups to the knot optimization of MARS, we need to modify the objective function of MARS that minimizes the $\mathrm{RSS}$ criterion. We need to add a term that penalizes the $\mathrm{RSS}$ difference between subgroups, which means that we need to add a term that increases the objective function value when the disparity is high, and decreases it when the disparity is low.  By minimizing the absolute error difference within the knot search objective function, we aim to ensure that the selection of predictor variables and their knots does not introduce or exacerbate bias or discrimination against any group. For example, if we use the absolute error difference as our measure of disparity, then we can add a term that is proportional to the absolute value of the error difference between subgroups.

To incorporate the $\mathrm{RSS}$ difference of subgroups into the knot optimization of MARS, we adjust the objective function used for minimizing the $\mathrm{RSS}$ criterion. This adjustment involves including a term that penalizes the $\mathrm{RSS}$ difference between subgroups. Essentially, this addition raises the objective function value when the disparity is high and reduces it when the disparity is low. The objective of minimizing the absolute error difference within the knot search function is to ensure that the selection of predictor variables and their knots does not introduce or exacerbate bias or discrimination against any specific group. 

The objective function value combines $\mathrm{RSS}$ and disparity weighted by $\lambda$, 
\begin{equation}\label{eq:fairness_equation}
\small
\begin{aligned}
    &\left(\sum_{q=1}^n\left(y_q-\bar{y}\right)^2-\sum_{i=1}^{M+1} \beta_i\left(c_i\right)\right) + \\
    \\
    &\lambda \left(\frac{1}{|\mathcal{S}|}\sum_{j=1}^{|\mathcal{S}|}\lvert \mathrm{RSS}_j-\mathrm{RSS}_{\mathcal{S}\backslash j}\rvert\right),
\end{aligned}
\end{equation}
where $\mathcal{S}$ is the set of subgroups for sensitive attribute $s$.

The goal is to ensure that in each iteration of the forward selection, while the model is minimizing $\mathrm{RSS}$ it also \mbox{minimizes} the absolute error difference between subgroups by selecting fair-aware knots.

The output is a MARS model optimized for fair knot locations and counts to minimize both $\mathrm{RSS}$ and disparity. Nevertheless, there may be a trade-off between accuracy and disparity, which can be adjusted using the $\lambda$ parameter based on specific preferences or constraints. 

Similarly, in the backward elimination step, the procedure involves calculating the change in both $\mathrm{RSS}$ and disparity by removing a pair of basis functions associated with an existing knot for each predictor variable. The objective is to select the predictor variable and knot that result in the least increase in both $\mathrm{RSS}$ and disparity.\\

\begin{lemma}
Consider a MARS model with $q$ basis functions and $\mathcal{M}$ set of coefficients, where $\mathrm{RSS}_q$ represents the residual sum of squares and $Disparity_q$ denotes the disparity of this model. Let $x_j$ be a predictor variable, and $k_{jl} \in \mathcal{K}_j$ be a candidate knot. Now, let $\mathrm{RSS}_{q+2}$ and $Disparity_{q+2}$ be the $\mathrm{RSS}$ and disparity for a new model, obtained by augmenting $\mathcal{M}$ with an additional pair of basis functions centered around $k_{jl}$. Define $\Delta \mathrm{RSS} = \mathrm{RSS}_{q+2} - \mathrm{RSS}_q$ and $\Delta Disparity = Disparity_{q+2} - Disparity_q$.

The alteration in the objective function, resulting from the addition of the basis functions, is given by:
\[ \Delta (\mathrm{RSS} + \lambda Disparity) = \Delta \mathrm{RSS} + \lambda \Delta Disparity \]

where $\lambda$ is a regularization parameter, influencing the trade-off between accuracy and disparity.

The objective is to minimize this alteration, leading to the identification of a knot that yields the smallest positive or most negative value for this change, thus favoring the selection of a knot that mitigates the unfairness of the MARS model's outcome.

\end{lemma}
\begin{proof}
Assume that we have a MARS model with $q$ basis functions and $\mathrm{RSS}_q$ and $Disparity_q$ as the $\mathrm{RSS}$ and disparity for this model. Let $x_j$ be a predictor variable and $k_{jl} \in \mathcal{K}j$ be a candidate knot. Let $\mathrm{RSS}_{q+2}$ and $Disparity_{q+2}$ be the $\mathrm{RSS}$ and disparity for the new model with the added pair of basis functions involving $k_{jl}$. We want to show that adding disparity to the objective function would result in selecting a knot that reduces the unfairness of the MARS outcome.\\
To do this, we need to compare the change in $\mathrm{RSS}$ and disparity by adding the pair of basis functions. Let $\Delta \mathrm{RSS} = \mathrm{RSS}_{q+2} - \mathrm{RSS}_q$ and $\Delta Disparity = Disparity_{q+2} - Disparity_q$. We can write the objective function as:
$\mathrm{RSS}(\mathcal{M}) + \lambda Disparity(\mathcal{M})$
where $\lambda$ is a regularization parameter that controls the trade-off between accuracy and disparity.
The change in objective function by adding the pair of basis functions is:
$\Delta (\mathrm{RSS} + \lambda Disparity) = \Delta \mathrm{RSS} + \lambda \Delta Disparity$.
We want to minimize this change, so we need to find a knot that gives the smallest positive or largest negative value for this change.
Now, suppose that there are two subgroups, $s_1$ and $s_2$, such that the MARS model is unfair to one of them. Without loss of generality, assume that the MARS model underpredicts for subgroup $s_1$ and overpredicts for subgroup $s_2$. This means that $\bar{y}_{s_1} < \bar{y}$ and $\bar{y}_{s_2} > \bar{y}$, where $\bar{y}_{s_i}$ is the mean predicted value for subgroup $s_i$ and $\bar{y}$ is the overall mean predicted value.
To reduce the unfairness of the MARS outcome, we need to find a knot that increases the $\mathrm{RSS}$ for subgroup $s_1$ and decreases the $\mathrm{RSS}$ for subgroup $s_2$. This means that we need to find a knot that increases the value of the basis function for subgroup $s_1$ and decreases the value of the basis function for subgroup $s_2$. This would result in a negative value for $\Delta Disparity$ since it would reduce the difference between subgroup {$\mathrm{RSS}$ and the overall $\mathrm{RSS}$.}
However, this may also result in a positive value for $\Delta \mathrm{RSS}$, since it may increase the error between actual and predicted values. Therefore, we need to balance the trade-off between accuracy and disparity by choosing an appropriate value for $\lambda$. If $\lambda$ is too small, then we may not reduce the unfairness enough. If $\lambda$ is too large, then we may sacrifice too much accuracy.
Therefore, by adding disparity to the objective function, we are more likely to select a knot that helps reduce the unfairness of the MARS outcome, as long as we choose an appropriate value for $\lambda$.
\end{proof}

\begin{proposition}
\textit{Let $x_j$ be a predictor variable and $\bar{x}_{j,s_i}$ be the mean value of $x_j$ for subgroup $s_i$. Then there exists a knot $k_{jl} \in \mathcal{K}_j$ such that $\bar{x}_{j,s_1} < k_{jl} < \bar{x}_{j,s_2}$ or $\bar{x}_{j,s_2} < k_{jl} < \bar{x}_{j,s_1}$ that increases the value of the basis function $B(x) = [x - k_{jl}]_+$ or $B(x) = [k_{jl} - x]_+$ for subgroup $s_1$ and decreases it for subgroup $s_2$.}
\end{proposition} 
\begin{proof} 
The proof is based on using the intermediate value theorem to find a knot between the mean values of $x_j$ for the two subgroups. The details are as follows:
Without loss of generality, assume that $\bar{x}_{j,s_1} < \bar{x}_{j,s_2}$. Let $f(x) = \bar{x}_{j,s_1} - x$. Then $f(\bar{x}_{j,s_1}) = 0$ and $f(\bar{x}_{j,s_2}) = \bar{x}_{j,s_1} - \bar{x}_{j,s_2} < 0$. Since $f(x)$ is continuous, by the intermediate value theorem, there exists a value $k_{jl} \in (\bar{x}_{j,s_1}, \bar{x}_{j,s_2})$ such that $f(k_{jl}) < 0$. This means that $\bar{x}_{j,s_1} < k_{jl} < \bar{x}_{j,s_2}$. Therefore, we can use the basis function $B(x) = [x - k_{jl}]_+$ to increase the value for subgroup $s_1$ and decrease the value for subgroup $s_2$. This is because subgroup $s_1$ will have positive values for this basis function (since their values are above $k_{jl}$), while subgroup $s_2$ will have zero values for this basis function (since their values are below $k_{jl}$).
\end{proof}
The proposition shows that we can change the shape and location of the MARS model by adding a pair of basis functions that involve a new or existing knot for each predictor variable. 
By choosing a knot that is between the mean values of $x_j$ for subgroup $s_1$ and subgroup $s_2$, we can increase the value of the basis function for subgroup $s_1$ and decrease the value of the basis function for subgroup $s_2$. This is because subgroup $s_1$ will have positive values for the basis function (since their values are above the knot), while subgroup $s_2$ will have zero values for the basis function (since their values are below the knot).\\
For example, suppose that we have a predictor variable $x_j$ with values ranging from 0 to 10, and two subgroups $s_1$ and $s_2$ with mean values of 3 and 7, respectively. If we choose a knot at 5, then we can use the basis function $B(x) = [x - 5]_+$ to increase the value for subgroup $s_1$ and decrease the value for subgroup $s_2$. This is because subgroup $s_1$ will have positive values for this basis function (since their values are above 5), while subgroup $s_2$ will have zero values for this basis function (since their values are below 5).
By increasing the value of the basis function for subgroup $s_1$ and decreasing the value of the basis function for subgroup $s_2$, we can change the predicted values for these subgroups. This means that the predicted values for subgroup $s_1$ will increase, and the predicted values for subgroup $s_2$ will decrease, resulting in a more balanced mean squared error ($\mathrm{MSE}$).
\subsection{Fairness-Aware Coefficient Estimation}
In \textit{fair}MARS, we enhance the loss function by incorporating weights that reflect the proportional distribution of observations within each protected value, transforming the MARS linear least-squares loss into a weighted least-squares regression \cite{agarwal2019fair}. The MARS model, developed by Friedman for $B$-spline fitting, is represented as:
\begin{equation}
\small
{g'}(\mathbf{x}) = \sum_{i=1}^{M-1} \beta_i B_i(\mathbf{x}) \text{+} \beta_M B_{m^*}(\mathbf{x}) x_v \text{+} \beta_{M\text{\tiny +}1} B_{m^*}(\mathbf{x})\left(x_v - k\right)_{\text{\tiny +}}
\end{equation}
The lack-of-fit (\textit{lof}) of $g'$ to the data is defined as:
\begin{equation}\label{eq:marsloss}
\small
\Delta[{g'}(\mathbf{x}), y_i]=\sum_{i=1}^{N} [{g'}({x_i})-y_i]^2
\end{equation}
MARS coefficients ($\beta_i$) are determined independently of response values ($y_1, \ldots, y_n$) and adjusted via a least-squares fitting procedure of $g'$ to the training dataset using basis function parameters.

For \textit{fair}coef estimation, we extend the loss function by \mbox{assigning} weights to observations based on $w_{i} = \frac{1}{\sigma_{i}}$, \mbox{reflecting} the proportion of each subgroup in the training data.
This transforms the MARS linear least-squares loss (equation \ref{eq:marsloss}) into a weighted least-squares regression:
\begin{equation}
 \small
\Delta[{g'}(\mathbf{x}), y_i]=\sum_{i=1}^{N} w_i[{g'}({x_i})-y_i]^2
\end{equation}
The \textit{fair}coef coefficients can be obtained through the \mbox{optimization} process:
\begin{equation}
\small
\left\{\hat{\beta}_j(\mathbf{x})\right\}_1^{M-1} = \underset{\left\{\beta_j\right\}_1^{M-1}}{\arg \min } \left(\sum_{i=1}^N w_i\left[{g'}({x_i})-y_i\right]^2\right)
\end{equation}
\section{Experiments}

\begin{table}[t]
    \centering
    \small
    \setlength{\tabcolsep}{1pt}
    \begin{tabular}{llccccc}
        \textbf{Dataset} & \textbf{Metric} & \textbf{MARS} & \textbf{\textit{fair}}\textbf{knot} & \textbf{DT} & \textbf{fairDT} & \textbf{fairLR} \\
        \hline
        ELS & $\text{$\mathrm{MSE}$}_{\text{test}}$ & 0.020 & 0.021 & 0.028 & 0.210 & 0.185 \\
        & $\text{Asian}$ & 0.005 & 0.004 & 0.006 & 0.063 & 0.103 \\
        & $\text{Black}$ & 0.016 & 0.016 & 0.006 & 0.029 & 0.062 \\
        & $\text{Hispanic}$ & 0.008 & 0.007 & 0.007 & 0.075 & 0.050 \\
        & $\text{Multiracial}$ & 0.022 & 0.023 & 0.040 & 0.009 & 0.134 \\
        & $\text{White}$ & 0.009 & 0.009 & 0.009 & 0.004 & 0.078 \\
        \hline
        Crime & $\text{$\mathrm{MSE}$}_{\text{test}}$ & 0.019 & 0.020 & 0.021 & 0.053 & 0.022 \\
        & $\text{Black}$ & 0.030 & 0.029 & 0.029 & 0.215 & 0.028 \\
        \hline
        Student & $\text{$\mathrm{MSE}$}_{\text{test}}$ & 3.666 & 3.574 & 10.62 & 14.127 & 4.631 \\
        Performance & $\text{Gender}$ & 3.665 & 3.487 & 6.715 & 5.571 & 3.233
    \end{tabular}
    \caption{$\mathrm{MSE}$ \& Subgroup Absolute Error Difference: \\ 10-fold Cross-Validation Across Models and Datasets}
    \label{tab:comparison}
\end{table}

\begin{table}[!ht]
\centering
\fontsize{9}{12}\selectfont
\setlength{\tabcolsep}{1pt}
\begin{tabular}{l|rrrrrrrr}
\diagbox{\textit{Metric}}{$\lambda$} & 0 \tiny(MARS) & 0.2 & 0.3 & 0.4 & 0.6 & 0.7 & 0.8 & 0.9 \\
\hline
$\text{$\mathrm{MSE}$}_{\text{test}}$ & 0.151 & 0.149 & 0.149 & 0.149 & 0.149 & 0.141 & 0.141 & 0.141 \\
$\text{$R^2$}_{\text{test}}$ & 0.687 & 0.691 & 0.691 & 0.691 & 0.691 & 0.708 & 0.708 & 0.708 \\
\hline
$\text{Asian}$ & 0.003 & 0.003 & 0.003 & 0.003 & 0.003 & 0.056 & 0.056 & 0.056 \\
$\text{Black}$ & 0.172 & 0.150 & 0.150 & 0.150 & 0.150 & 0.160 & 0.160 & 0.160 \\
$\text{Hispanic}$ & 0.017 & 0.012 & 0.012 & 0.012 & 0.012 & 0.012 & 0.012 & 0.012 \\
$\text{Multiracial}$ & 0.056 & 0.057 & 0.057 & 0.057 & 0.057 & 0.050 & 0.050 & 0.050 \\
$\text{White}$ & 0.058 & 0.048 & 0.048 & 0.048 & 0.048 & 0.031 & 0.031 & 0.031 \\
\end{tabular}
\caption{$\mathrm{MSE}$, $R^2$, \& Subgroup Absolute Error Difference: Single Fold Implementation of \textit{fair}knot on ELS Dataset}
\label{tab:knot1}
\end{table}

\begin{table*}[!ht]
    \centering
    \fontsize{9}{12}\selectfont
    \setlength{\tabcolsep}{1pt}
    \begin{tabular}{cccccc}
        \textbf{\textit{fair}\text{coef}} & \textbf{\text{MARS}} & & \textbf{\textit{fair}\text{knot \& }\textit{fair}{coef}} & \textbf{\textit{fair}\text{knot}} & \\
        \hline
        \textbf{\textit{coef}} & \textbf{\textit{coef}} & {\textbf{Basis Function}} & \textbf{\textit{coef}} & \textbf{\textit{coef}} & {\textbf{Basis Function}} \\
        \hline
        0.20 & 0.18 & \textit{Intercept} & 0.67 & 0.69 & \textit{Intercept} \\
        \hline
        0.53 & 0.51 & F3\_GPA(first year) & 0.58 & 0.53 & h(F3\_GPA(first year) \\
        \hline
        0.06 & 0.10 & 9\_12\_GPA & 0.07 & 0.11 & 9\_12\_GPA \\
        \hline
        \text{\tiny\textit{pruned}} & \text{\tiny\textit{pruned}} & h(Std Math/Reading-49.86) & \text{\tiny\textit{pruned}} & \text{\tiny\textit{pruned}} & h(Std Math/Reading-49.08) \\
        \hline
        -0.03 & -0.02 & h(49.86-Std Math/Reading) & \text{\tiny\textit{pruned}} & \text{\tiny\textit{pruned}} & h(49.08-Std Math/Reading) \\
        \hline
        0.11 & 0.09 & Ins attended  & 0.11 & 0.09 & Ins attended \\
        \hline
       0.12 & 0.11 & gender\_Female & 0.12 & 0.09 & gender\_Female\\
        \hline
         0.01  & 0.02 & Parents education & \text{\tiny\textit{pruned}} & \text{\tiny\textit{pruned}} & h(F1 Std Math-52.65) \\
        \hline
        -0.18 & -0.16 & Black  & \text{\tiny\textit{pruned}} & \text{\tiny\textit{pruned}} & h(52.65-F1 Std Math) \\
        \hline
        0.19 & 0.15 & h(\%Hispanic teacher-93) & 0.02 & 0.02 & Parents education  \\
        \hline
        \text{\tiny\textit{pruned}} & \text{\tiny\textit{pruned}} & h(93-\%Hispanic teacher) & \text{\tiny\textit{pruned}} & \text{\tiny\textit{pruned}} & h(Std Math/Reading-59.75)  \\
        \hline
        0.09 & 0.06  & Generation & -0.02  & -0.01  & h(59.75-Std Math/Reading) \\
        \hline
        0.01 & 0.01 & credits(first year) & & & \\
        \hline
        \text{\tiny\textit{pruned}} & \text{\tiny\textit{pruned}} & Marital\_Par\_Single   & & & \\
        \hline
        \text{\tiny\textit{pruned}} & \text{\tiny\textit{pruned}} & Marital\_Par\_Married    & & & \\
        \hline
        0.01 & 0.01 & F1\_Std\_Math & & & \\
        \hline
        \text{\tiny\textit{pruned}} & \text{\tiny\textit{pruned}} & \%Indian teacher  & & & \\
    \end{tabular}
    \caption{Comparison of MARS, \textit{fair}knot, \textit{fair}coef, and \textit{fair}knot\&\textit{fair}coef}
    \label{tab:comparison_table3}
\end{table*}



\paragraph{Datasets:} Our evaluations employed diverse real-world datasets to comprehensively assess the performance and fairness attributes of the proposed \textit{fair}MARS model. Dataset details are outlined as follows: \textbf{ELS Dataset} \cite{bozick2007education}: Comprising 1186 instances and featuring 40 predictor variables, we consider ``race'' as a sensitive attribute with 5 distinct categories and ``GPA'' as the response variable. \textbf{Crime Dataset} \cite{uci_ml_repository}: Consisting of 1994 instances and 128 predictor variables, we consider ``race'' as a sensitive attribute, classified into 2 categories: 1 representing a predominantly black community, and 0 denoting others. The response variable is ``crimes per capita''. The analysis included the top 11 important \mbox{features} using random forest feature importance scores. \textbf{Student Performance Dataset} \cite{student_performance_dataset}: Comprising 395 instances and 33 predictor variables, we consider ``gender'' as a sensitive attribute, classified with a binary distinction. The response variable is the final grade.

\paragraph{Evaluated Algorithms:} \textbf{MARS}: The original MARS algorithm without fairness enhancements \cite{friedman1991multivariate}. \textbf{\textit{fair}MARS}: In our analysis, we will assess the individual components of \textit{fair}MARS, namely \textit{fair}knot and \textit{fair}coef, both independently and in combination. This evaluation will be conducted in comparison to the MARS algorithm. We will further compare the performance of \textit{fair}knot against the following algorithms: \textbf{MIP-DT fair regression tree} \cite{Aghaei_Azizi_Vayanos_2019}: This approach integrates a \mbox{regularization} term into the Mixed-Integer Programming objective function to address fairness (disparate impact and treatment) in regression tasks. In this paper, we will denote this method as ``fairDT''. This approach aims to balance accuracy and fairness, albeit with higher computation times.
\textbf{MIP regression tree (CART)}: a classic method for constructing decision trees in regression problems. This method does not consider any fairness criteria. We will denote this method as ``DT.''
\textbf{Fair Regression} \cite{agarwal2019fair}: This approach leverages bounded group loss in \mbox{regression} tasks. Employing Linear Regression as its \mbox{estimator}, it aims to minimize empirical error while reducing unfair disparities among distinct groups. We will denote this method as ``fairLR.''

\paragraph{Implementation:} We implemented \textit{fair}MARS by extending the Py-earth package \cite{rudy2016py}, a Python implementation of \citet{friedman1991multivariate}'s MARS algorithm following the Scikit-learn design. Py-earth uses Cython to optimize its implementation and supports Scikit-learn's interfaces. Our extended version preserves the original capabilities of the Py-earth package, including efficient handling of large datasets and fast processing, while incorporating additional components to embed fairness considerations.
\textit{fair}MARS resources, including code, instructions, datasets and variable descriptions, can be found on GitHub\footnote{https://github.com/parianh/fairMARS} \cite{parian_haghighat_2024_10595371}.
For the fairDT, we utilized \citet{Aghaei_Azizi_Vayanos_2019}'s MIP-DT code, employing Gurobi 10.0.2 on a computer node with 18 CPUs and 256 GB of RAM. For the DT method, we used the MIP-DT code, setting the fairness penalty value to 0 to deactivate the fairness component. For the fairLR approach, we used the GridSearchReduction method from the aif360 package, employing linear regression as the estimator and squared loss for optimization.
 
\subsection{Fairness-Accuracy Trade-off Analysis} Exploring the interplay between fairness and accuracy, we conducted evaluations using $\lambda \in \{0.2, 0.4, 0.6, 0.8\}$ within both \textit{fair}MARS and fairDT models. Table \ref{tab:comparison} provides insights into $\mathrm{MSE}$ and subgroup disparity across datasets for these models. We selected the optimal $\lambda$ for each model to minimize discrimination. Notably, \textit{fair}MARS demonstrated remarkable efficiency, with average fold times of about 3, 2, and 0.5 minutes for the ELS, crime, and student performance datasets, respectively. In contrast, the fairDT model took 3, 2, and 1 hour per fold to optimize on the ELS, crime, and student performance datasets, respectively.


\begin{figure}[!ht]
    \centering
    \includegraphics[width=1\linewidth]{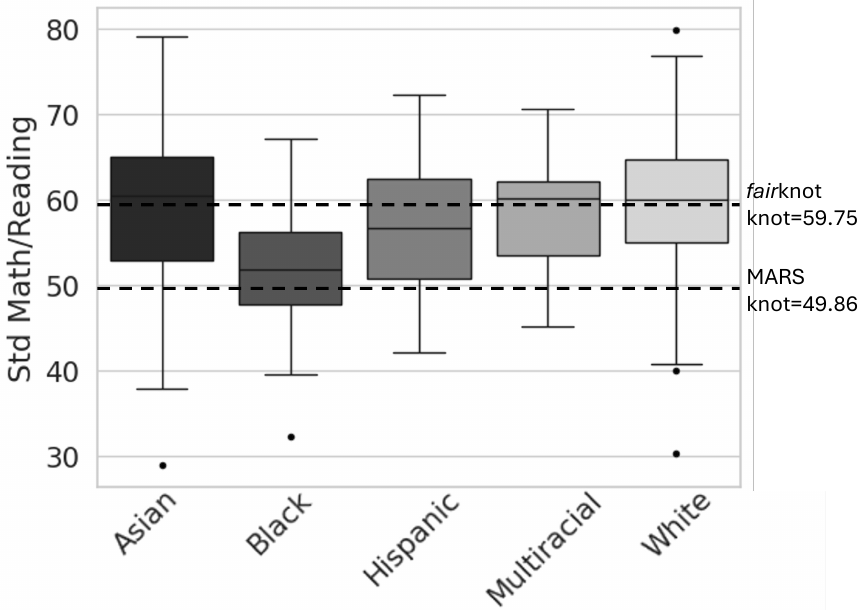}
    \caption{Knot Selection \textit{fair}knot vs. MARS}
    \label{fig:arrangement}
\end{figure}

\subsection{Fair Knot Selection vs. Fair Coefficient Estimation} In this section, we present a practical case study where we apply MARS and our proposed \textit{fair}MARS with different subroutines for fairness considerations on a single fold of the ELS dataset. We investigate \textit{fair}MARS with three subroutines: using \textit{fair}knot, \textit{fair}coef, and a combination of both \textit{fair}knot and \textit{fair}coef. The goal is to compare the resulting basis functions and coefficients across these models and analyze their effects on fairness and accuracy.
In Table \ref{tab:comparison_table3}, we can observe that \textit{fair}coef primarily impacts the coefficients, aiming to enforce fairness while keeping the basis functions unchanged.
On the other hand, \textit{fair}knot leads to different sets of basis functions and knot positions. Some other observations from this comparison are:\\
\noindent \textbf{Balancing Fairness and Accuracy:} MARS is a flexible and adaptive technique that can adjust its basis functions according to the data distribution \cite{friedman1991multivariate}. While being sensitive to data distribution, MARS maintains its accuracy due to its inherent generalized cross-validation mechanism. This adaptability of MARS proves to be a remarkable advantage in the context of \textit{fair}MARS, where the interplay between fairness and accuracy does not necessarily lead to a trade-off. This intriguing observation is supported by the results presented in Table \ref{tab:knot1}, where the $\mathrm{MSE}$ and $R^2$ metrics remain not only consistent but occasionally even exhibit improvement.\\
\indent It is also important to emphasize that \textit{fair}knot introduces localized corrections to reduce the error differences between subgroups, such as white and black students, choosing different knot locations for each subgroup. This makes the regression model more fair and less biased towards any subgroup. However, this does not guarantee that the error of each subgroup will always go down. Depending on the data and the $\lambda$ value, some subgroups may have the same or slightly higher error when we use \textit{fair}knot, while others may have lower error.
For instance, in Table \ref{tab:knot1}, the disparity for the Asian subgroup remains stable or slightly worsens with the integration of \textit{fair}knot, while the disparities of other subgroups show improvement. This is because \textit{fair}knot balances fairness and accuracy, and there is sometimes a trade-off between them.
\\
\textbf{\textit{fair}knot Impact on Basis Functions:} Table \ref{tab:comparison_table3} showcases a specific instance in which \textit{fair}knot modifies the knot location of a basis function: ``Std Math/Reading''. This continuous predictor exhibits varying distributions across different subgroups (as seen in Figure \ref{fig:arrangement}). \textit{fair}knot adjusts the knot position closer to the average distribution of all subgroups. This demonstrates that \textit{fair}knot can adapt the basis functions to the data characteristics of each subgroup, ultimately enhancing fairness. Table 5 presents a comparative analysis of the results obtained using MARS and \textit{fair}MARS on higher-degree (2nd-degree polynomial) basis function and coefficient estimates. The MARS results reveal 9 higher-order \mbox{interactions}, while \textit{fair}MARS reduces this count to 5 \mbox{interactions}. We also observe that the \textit{fair}MARS \mbox{maintains} the selection of the most important features in the ELS dataset [9\_12\_GPA, F3\_GPA(fist year), credits(first year), Std Math/Reading].\\
\begin{table}[!ht]
    \centering
    \fontsize{9}{12}\selectfont
    \setlength{\tabcolsep}{1pt}
    \begin{tabular}{c|c}
        \multicolumn{2}{c}{\textbf{\text{MARS}}} \\
        \hline
        \textbf{\textit{coef}} & \textbf{Basis Function} \\
        \hline
        -0.64 & (Intercept) \\
        \hline
        0.64 & F3\_GPA(first year) \\
        \hline
        \text{\tiny\textit{pruned}} & 9\_12\_GPA \\
        \hline
        -0.01 & h(Std Math/Reading-49.86)*9\_12\_GPA \\
        \hline
        \text{\tiny\textit{pruned}} & h(49.86-Std Math/Reading)*9\_12\_GPA \\
        \hline
        0.01 & credits(first year)*F3\_GPA(first year) \\
        \hline
        -0.01 & credits(first year)*9\_12\_GPA \\
        \hline
        0.22 & Ins\_attended \\
        \hline
        0.01 & \%white teacher \\
        \hline
        -0.01 & \%white teacher*F3\_GPA(first year) \\
        \hline
        \text{\tiny\textit{pruned}} & gender\_Female \\
        \hline
        0.01 & Parents education*\%white teacher \\
        \hline
        \text{\tiny\textit{pruned}} & h(Std Math/Reading-52.23)*gender\_Female \\
        \hline
        \text{\tiny\textit{pruned}} & h(52.23-Std Math/Reading)*gender\_Female \\
        \hline
        -0.01 & credits(first year)*Ins\_attended \\
        \hline
        0.01 & \%Hispanic\_NR teacher*\%white teacher \\
        \hline
        \text{\tiny\textit{pruned}} & \%Hispanic\_NR teacher*F3\_GPA(first year) \\
        \hline
        0.01 & Std Math/Reading*9\_12\_GPA \\
        \hline
        \text{\tiny\textit{pruned}} & h(Std Math/Reading-68.48)*9\_12\_GPA \\
        \hline
        \text{\tiny\textit{pruned}} & h(68.48-Std Math/Reading)*9\_12\_GPA \\
        \hline
        \text{\tiny\textit{pruned}} & Marital\_Par\_Single\_chld*Ins\_attended \\
        \hline
        -0.02 & Parents education*F3\_GPA(first year) \\
        \hline
        0.03 & English*F3\_GPA(first year) \\
        \hline
        \text{\tiny\textit{pruned}} & English*9\_12\_GPA \\
        \hline
        0.00 & \%Black teacher*gender\_Female \\
        \hline
        0.16 & Hispanic \\
        \hline
        \text{\tiny\textit{pruned}} & \%Hawaiian teacher*F3\_GPA(first year)\\
         \hline
        \text{\tiny\textit{pruned}} & \%Indian teacher*\%white teacher\\
        \hline
        \text{\tiny\textit{pruned}} & White*F3\_GPA(first year) \\
        \\
        \textbf{\textit{}} & \textbf{\textit{fair}knot} \\
        \hline
        \textbf{\textit{coef}} & \textbf{Basis Function} \\
        \hline
        0.70 & (Intercept) \\
        \hline
        0.33 & F3\_GPA(first year) \\
        \hline
        0.24 & 9\_12\_GPA \\
        \hline
        0.00 & h(Std Math/Reading-49.08)\\
        \hline
        -0.03 & h(49.08-Std Math/Reading)\\
        \hline
        0.01 & credits(first year)*F3\_GPA(first year) \\
        \hline
        -0.01 & credits(first year)*9\_12\_GPA \\
        \hline
        \text{\tiny\textit{pruned}} & h(Std Math/Reading-59.75)*h(Std Math/Reading-49.08)\\ 
        \hline
        \text{\tiny\textit{pruned}} & h(59.75-Std Math/Reading)*h(Std Math/Reading-49.08)\\
    \end{tabular}
    \caption{Comparative Analysis of Non-Linear Modeling (2nd Degree Polynomial) of MARS and \textit{fair}knot.}
    \label{tab:coefficients_table}
\end{table}
\\
These results demonstrate that our proposed \textit{fair}MARS approach can balance fairness and accuracy in regression modeling by using different combinations of fair knot selection and fair coefficient estimation. Depending on the application domain and the desired trade-off, users can choose the appropriate scenario for their regression tasks.

\section{Conclusion}
In this study, we introduced the \textit{fair}MARS algorithm, which incorporates fairness metrics into the learning process. Our primary objective was to address disparities among sensitive subgroups by optimizing knots, a procedure supported by both theoretical justification and empirical validation. The \textit{fair}MARS algorithm not only mitigates disparity but also produces interpretable decision rules by \mbox{deriving} \mbox{optimal} and fair variable splitting criteria. The algorithm's high computational speed makes it particularly advantageous for practitioners seeking efficient and equitable predictive modeling solutions in a regression setting.



\section{Acknowledgements}

The authors extend sincere appreciation for the generous support provided by the U.S. Department of Education's Institute of Education Sciences.
Parian Haghighat, Hadis Anahideh, and Denisa G\'andara are supported by the Institute of Education Sciences
R305D220055 grant. Lulu Kang is supported in part by NSF DMS 2153029 and DMS 1916467 grants.

\bigskip

\bibliography{aaai24.bib}

\end{document}